\documentclass[10pt,letterpaper]{article}

\usepackage{times}
\usepackage{epsfig}
\usepackage{graphicx}
\usepackage{amsmath}
\usepackage{amssymb}
\usepackage[utf8]{inputenc}
\usepackage{algorithm}
\usepackage{amsthm}
\usepackage{algpseudocode} 
\usepackage[numbers]{natbib}
\usepackage{authblk}

\newcommand{\BEAS}{\begin{eqnarray*}}
\newcommand{\EEAS}{\end{eqnarray*}}
\newcommand{\BEA}{\begin{eqnarray}}
\newcommand{\EEA}{\end{eqnarray}}
\newcommand{\BEQ}{\begin{equation}}
\newcommand{\EEQ}{\end{equation}}
\newcommand{\BIT}{\begin{itemize}}
\newcommand{\EIT}{\end{itemize}}
\newcommand{\BNUM}{\begin{enumerate}}
\newcommand{\ENUM}{\end{enumerate}}
\newcommand{\BA}{\begin{array}}
\newcommand{\EA}{\end{array}}

\DeclareMathAlphabet{\mathpzc}{OT1}{pzc}%
                                 {m}{it}
\usepackage{natbib}

\DeclareMathOperator{\argmax}{argmax}

\DeclareMathOperator{\Tr}{Tr}
\DeclareMathOperator{\tr}{Tr}

\DeclareMathOperator{\Diag}{Diag}

\newcommand{\croch}[1]{\left[ #1 \right] }
\newcommand{\set}[1]{\left\{ #1 \right\} }

\newtheorem{theorem}{Theorem}

\title{Large-Margin Metric Learning for Partitioning Problems}
\author[1,2]{R\'emi Lajugie\thanks{remi.lajugie@ens.fr}}
\author[1,2]{Sylvain Arlot\thanks{sylvain.arlot@ens.fr}}
\author[1,2]{Francis Bach\thanks{francis.bach@ens.fr}}

\affil[1]{D\'epartement d'Informatique, Ecole Normale Sup\'erieure, Paris, France}
\affil[2]{INRIA, Equipe projet SIERRA}

\begin{document}

\maketitle

\begin{abstract}

In this paper, we consider unsupervised partitioning problems, such as clustering, image segmentation, video segmentation and other change-point detection problems. We focus on partitioning problems based explicitly or implicitly on the minimization of Euclidean distortions, which include mean-based change-point detection, K-means, spectral clustering and normalized cuts. Our main goal is to learn a Mahalanobis metric for these unsupervised problems, leading to feature weighting and/or selection. This is done in a supervised way by assuming the availability of several potentially partially labelled datasets that share the same metric. We cast  the metric learning problem as a large-margin structured prediction problem, with proper definition of regularizers and losses, leading to a convex optimization problem which can be solved efficiently with iterative techniques.
We provide experiments where we show how learning the metric may significantly improve the partitioning performance in synthetic examples, bioinformatics, video segmentation and image segmentation problems.
 

\end{abstract}

\section{Introduction} 
\label{Intro}

Unsupervised partitioning problems are ubiquitous in machine learning and other data-oriented fields such as computer vision, bioinformatics or signal processing. They include (a) traditional \emph{unsupervised clustering} problems, with the classical K-means algorithm, hierarchical linkage methods~\cite{gower1969minimum} and spectral clustering~\cite{ng2002spectral}, (b) \emph{unsupervised image segmentation} problems where two neighboring pixels are encouraged to be in the same cluster, with mean-shift techniques~\cite{cheng1995mean} or normalized cuts~\cite{ShiMalik}, and (c) \emph{change-point detection} problems adapted to multivariate sequences (such as video) where segments are composed of contiguous elements, with typical window-based algorithms~\cite{desobry2005online} and various methods looking for a change in the mean of the features~(see, e.g.,~\cite{chen2011parametric}).

All the algorithms mentioned above rely on a specific distance (or more generally a similarity measure) on the space of configurations. A good metric is crucial to the performance of these partitioning algorithms and its choice is heavily problem-dependent. While the choice of such a metric has been originally tackled manually (often by trial and error), recent work has considered learning such metric directly from data. Without any supervision, the problem is ill-posed and methods based on generative  models may learn a metric or reduce dimensionality (see, e.g.,~\cite{de2006discriminative}), but typically with no  guarantees that they lead to better partitions. In this paper, we follow~\cite{bar2006learning,Xing,BachJo} and consider the goal of learning a metric for potentially several partitioning problems sharing the same metric, assuming that several fully or partially labelled partitioned datasets are available during the learning phase. While such labelled datasets are typically expensive to produce, there are several scenarios where these datasets have already been built, often for evaluation purposes. These occur in video segmentation tasks (see Section~\ref{sec:video}), image segmentation tasks (see Section~\ref{sec:image}) as well as change-point detection tasks in bioinformatics (see  \cite{Toby} and Section~\ref{sec.appli.change-distrib}).

In this paper, we consider partitioning problems based explicitly or implicitly on the minimization of Euclidean distortions, which include K-means, spectral clustering and normalized cuts, and mean-based change-point detection. We make the following contributions:

\begin{list}{\labelitemi}{\leftmargin=1.1em}
   \addtolength{\itemsep}{-.0\baselineskip}
\item[--] We review and unify several partitioning algorithms in Section~\ref{sec:partitions}, and cast them as the maximization of a linear function of a rescaled equivalence matrix, which can be solved by algorithms based on spectral relaxations or dynamic programming. 
\item[--] Given fully labelled datasets, we cast in Section~\ref{sec:structure} the metric learning problem as a large-margin structured prediction problem, with proper definition of regularizers, losses and efficient loss-augmented inference.

\item[--] Given partially labelled datasets, we propose in Section~\ref{sec:extensions} an  algorithm, iterating between labelling the full datasets given a metric and learning a metric given the fully labelled datasets. We also consider in Section~\ref{sec:dist} extensions that allow changes in the full distribution of univariate time series (rather than changes only in the mean), with application to bioinformatics.

\item[--] We provide in Section~\ref{sec:experiments} experiments where we show how learning the metric may significanty improve the partitioning performance in synthetic examples, video segmentation  and image segmentation problems.
\end{list}

\vspace*{-.25cm}

\paragraph{Related work.} \hspace*{.000cm}

The need for metric learning goes far beyond unsupervised partitionning problems. \cite{Weinberger} proposed a large margin framework for learning a metric in nearest-neighbours algorithms based on sets of must-link/must not link constraints, while \cite{goldberger2004neighbourhood} considers a probability-based non-convex formulation. For these works, a single dataset is fully labelled and the goal is to learn a metric leading to good testing performance on unseen data.

Some recent work \cite{JainKulis} proved links between metric learning and kernel learning, permitting to kernelize any Mahalanobis distance learning problem.

Metric learning has also been considered in semi-supervised clustering of a single dataset, where some partial constraints are given. This includes the works of~\cite{bar2006learning,Xing}, both based on efficient convex formulations. As shown in Section~\ref{sec:experiments}, these can be used in our settings as well by stacking several datasets into a single one. However, our discriminative large-margin approach outperforms these.

Moreover, the task of learning how to partition was tackled in~\cite{BachJo} for spectral clustering. The problem set-up is the same (availability of several fully partitioned datasets), however, the formulation is non-convex and relies on the unstable optimization of eigenvectors. In Section~\ref{sec:nc}, we propose a convex more stable large-margin approach. 

Other approaches do not require any supervision \cite{de2006discriminative}, and perform dimensionality reduction and clustering at the same time, by iteratively alternating the computation of a low-rank matrix and a clustering of the data using the corresponding metric. However, they are unable to take advantage of the labelled information that we use.  

Our approach can also be related to the one of \cite{szummer2008learning}. Given a small set of labelled instances, they use a similar large-margin framework, inspired by \cite{Tsochan} to learn parameters of Markov random fields, using graph cuts for solving the ``loss-augmented inference problem'' of structured prediction. However, their segmentation framework does not apply to unsupervised segmentation (which is the goal of this paper). In this paper, we present a supervised learning framework aiming at learning how to perform an unsupervised task.

Our approach to learn the metric is nevertheless slightly different of the ones mentioned above. Indeed, we cast this problem as the solution of a structured SVM as in \cite{Tsochan, Taskar}. This make our paper shares many conceptual steps with works like \cite{Caetano:Graph,McFeeGert} where they use a structured SVM to learn in one case weights for graph matchings and a metric for ranking in the other case.

\section{Partitioning through matrix factorization}
\label{sec:partitions}

In this section, we   consider $T$ multi-dimensional observations $x_1,\dots,x_T \in \mathbb{R}^P$, which may be represented in a matrix $X \in \mathbb{R}^{T \times P}$. Partitioning the $T$ observations into $K$ classes is equivalent to finding an \emph{assignment matrix}  $Y \in \{0,1\}^{T \times K}$, such that  $Y_{ij}=1$ if the $i$-th observation is affected to cluster $j$ and $0$ otherwise. For general partitioning problems, no additional constraints are used, but for change-point detection problems, it is assumed that the segments are contiguous and with increasing labels. That is, the matrix $Y$ is of the form
\[
Y =\begin{pmatrix}
\mathbf{1}_{T_1}    & 0 & \ldots & 0 \\
0 & \ddots & \ddots & \vdots \\
\vdots & \ddots  & \ddots & 0 \\
0 & \ldots & 0 & \mathbf{1}_{T_K}  \\
\end{pmatrix},
\]
where $\mathbf{1}_D \in \mathbb{R}^D$ is the $D$-dimensional vector with constant components equal to one, and $T_j$ is the number of elements in cluster $j$. For any partition, we may re-order (non uniquely) the data points so that the assignment matrix has the same form; this is typically useful for the understanding of partitioning problems.

\subsection{Distortion measure}
In this paper, we consider partitioning models where each data point in cluster $j$ is modelled by a vector (often called a centroid or a mean) $c_j \in \mathbb{R}^p$, the overall goal being to find a partition and a set of means so that the distortion measure $\sum_{i=1}^T \sum_{j=1}^K Y_{ij} \| x_i - c_j\|^2$ is as small as possible, where $\| \cdot \|$ is the Euclidean norm in $\mathbb{R}^P$. By considering the Frobenius norm defined through $\|A\|_F^2 = \sum_{i=1}^T \sum_{j=1}^P A_{ij}^2$, this is equivalent to minimizing
\begin{equation}
\label{eq:distortion}
\| X - Y C \|_F^2
\end{equation}
with respect to an assignment matrix $Y$  and the centroid matrix $C \in \mathbb{R}^{K \times P}$. 

\subsection{Representing partitions}
Following~\cite{BachJo,de2006discriminative}, the quadratic minimization problem in $Y$ can be solved in closed form, with solution $C = (Y^\top Y)^{-1} Y^\top X$ (it can be found by computing the matrix gradient and setting it to zero). Thus, the partitioning problem (with known number of clusters $K$) of minimizing the distortion in Eq.~(\ref{eq:distortion}), is equivalent to:
\begin{equation}
\label{eq:distortionM}
\min_{ Y \in \{0,1\}^{T \times K}, \ Y\mathbf{1}_K = \mathbf{1}_P }
\| X - Y (Y^\top Y)^{-1} Y^\top X \|_F^2 \enspace .
\end{equation}
Thus, the problem is naturally parameterized by the $T \times T$-matrix $M = Y (Y^\top Y)^{-1} Y^\top$. This matrix, which we refer to as a \emph{rescaled equivalence matrix}, has a specific structure. First the matrix $Y^\top Y$ is diagonal, with $i$-th diagonal element equal to the number of elements in the cluster containing the $i$-th data point. Thus $M_{ij} = 0$ if $i$ and $j$ are in different clusters and otherwise equal to $1/D$ where $D$ is the  number of elements in the cluster containing the $i$-th data point. Thus, if the points are re-ordered so that the segments are composed of contiguous elements, then we have the following form
\[
M =\begin{pmatrix}
\mathbf{1} \mathbf{1}^{\top} / T_1 & 0 & \ldots & 0 \\
0 & \ddots & \ddots & \vdots \\
\vdots & \ddots  & \ddots & 0 \\
0 & \ldots & 0 & \mathbf{1} \mathbf{1}^{\top} / T_K \\
\end{pmatrix}.
\]
In this paper, we  use this representation of partitions. Note the difference with alternative representations $YY^\top$ which has values in $\{0,1\}$, used in particular by~\cite{JouBacPonc10_cvpr}.

We  denote by $\mathcal{M}_K$ the set of rescaled equivalence matrices, i.e., matrices $M \in \mathbb{R}^{T \times T}$ such that there exists an assignment matrix $Y \in \mathbb{R}^{ T \times K}$ such that $M = Y (Y^\top Y)^{-1} Y^\top$. For situations where the number of clusters is unspecified, 
we denote by $\mathcal{M}$ the union of all $\mathcal{M}_K$ for $K \in \{1,\dots,N\}$.

Note that the number of clusters may be obtained from the trace of $M$, since $\tr M = \tr Y (Y^\top Y)^{-1} Y^\top = 
\tr  (Y^\top Y)^{-1} Y^\top Y  = K$. This can also be seen by noticing that 
$M^2 = Y (Y^\top Y)^{-1} Y^\top Y (Y^\top Y)^{-1} Y^\top = M $, i.e., $M$ is a projection matrix, with eigenvalues in $\{0,1\}$, and the number of eigenvalues equal to one is exactly the number of clusters.
 Thus, 
$
\mathcal{M}_K = \big\{ M \in \mathcal{M}, \ \tr M = K \big\}.
$

\vspace*{-.25cm}

\paragraph{Learning the number of clusters $K$.} \hspace*{.000cm}
Given the number of clusters $K$, we have seen from Eq.~\eqref{eq:distortionM} that the partitioning problem is equivalent to
\begin{equation}
\label{eq:partition}
\min_{M \in \mathcal{M}_K} \| X - M X \|_F^2 = \min_{M \in \mathcal{M}_K}  \tr \big[ XX^\top  ( I - M)  \big].
\end{equation}
In change-point detection problems, an extra constraint of contiguity of segments is added.

In the common situation when the number of clusters~$K$ is unknown, then it may be estimated directly from data by penalizing the distortion measure by a term proportional to the number of clusters, as usually done for instance in change-point detection~\cite{Lav:2005}. 
This is a classical idea that can be traced back to the AIC criterion~\cite{Akaike} for instance. 
Given that the number of clusters for a rescaled equivalence matrix $M$ is $\tr M$, this leads to the following formulation:
\begin{equation}
\label{eq:partition2}
 \min_{M \in \mathcal{M}} \tr \big[ XX^\top  ( I - M)  \big] + \lambda \tr M
\end{equation}
Note that our metric learning algorithm also learns this extra parameter $\lambda$.

Thus,  the two types of partitioning problems (with fixed or unknown number of clusters) can be cast as the problem of maximizing a linear function of the form $\tr ( AM)$ with respect to $M \in \mathcal{M}$, with the potential constraint that $\tr M = K$. In general, such optimization problems may not be solved in polynomial time. In Section~\ref{sec:cp}, we show how adding contiguity constraints makes it possible to obtain a solution in polynomial time through dynamic programming. For general situations, the $K$-means algorithm, although not exact, can be used to get good partitioning in polynomial time. In Section~\ref{sec:spectral}, we provide a spectral relaxation, which we  use within our large-margin framework in Section~\ref{sec:structure}.

\subsection{Change-point detection by dynamic programming}
\label{sec:cp}
\label{sec:dp}
The change-point detection problem is a restriction of the general partitioning problem where the segments are composed of contiguous elements. We denote by $\mathcal{M}^{\rm seq}$ the set of partition matrices for the change-point detection problem, and $\mathcal{M}_K^{\rm seq}$, its restriction to partitions with $K$ segments.

The problem is thus of solving  Eq.~(\ref{eq:partition2}) (known number of clusters)
or Eq.~(\ref{eq:partition})  (unknown number of clusters) with the extra constraint that
$M \in \mathcal{M}^{\rm seq}$. 
In these two situations, the contiguity constraint leads to \emph{exact} polynomial-time algorithms based on dynamic programming. See, e.g.,~\cite{Rigaill}. This leads to algorithms for maximizing $\tr (AM)$, when $A$ is positive semi-definite in  $O(T^2)$. When the number of segments $K$ is known the running time complexity is  $O(KT^2)$. 

We now describe a reformulation that can solve $\max_{M \in \mathcal{M}}\Tr(A M)$ for any matrix $A$ (potentially with negative eigenvalues, as from Eq.~(\ref{eq:partition2})).
This algorithm is presented in Algorithm~\ref{alg3}. 
It only requires some preprocessing of the input matrix $A$, namely computing its summed area table $I$  (or image integral), defined to have the same size as $A$ and  with $I_{ij}=\sum_{i' \leq i, \ j' \leq j}A_{i'j'}$. In words it is the sum of the elements of $A$ which are above and to the left of respectively $i$ and $j$.
A similar algorithm can be derived in the case where $M \in \mathcal{M}_K$.
\begin{algorithm}[h]   
\caption{Dynamic programming for maximizing $\Tr(AM)$ such that $M\in \mathcal{M} $}          
\label{alg3}                           
\begin{algorithmic}                    
    \Require $T \times T$  matrix $A$
    \State Compute $I$, image integral (summed area table) of  $A$
    \State  Initialize $C(1,:)= $diag$(I)$

    \For{$t=1:T-1$}
    	
		\State      $C(t+1,t+1)= \max (C(1:t,t))+I(t+1,t+1)$
		\For{u=t+1 \ldots T}     
		\State $\beta=\frac{ I(s,s)+I(t+1,t+1)-I(s,t+1)-I(t+1,s)}{(u-t)}$
           		\State $C(t+1,u)= \max(C(1:t,t))+\beta$
  		\EndFor
 
    \EndFor \\
    Backtracking steps:    $t_c =T$, $Y=\emptyset$
        \While {$t_c \geqslant 1$}
 					\State $t_{c}^{\rm old}=t_c$,  ${t_c}= \argmax \set{ C({t_c},:) }$
            		\State $s=t_c^{\rm old}-t_{c}+1$, $Y=\begin{pmatrix}
        		Y & 0 \\
        		0 & \bold{1}_s
        \end{pmatrix}$  
        
        \vspace*{-.25cm}

			\State     
        \EndWhile \\
        
\Return Matrix $M=Y(Y^{\top} Y)^{-1} Y^{\top}$.
\end{algorithmic}
\end{algorithm}

\subsection{K-means clustering and spectral relaxation}
\label{sec:spectral}
For a known number of clusters $K$, 
K-means  is an iterative algorithm aiming at minimizing the distortion measure in Eq.~(\ref{eq:distortion}): it iterates between (a) optimizing with respect to $C$, i.e., $C = (Y^\top Y)^{-1} Y^\top X$, and (b) minimizing with respect to $Y$ (by assigning points to the closest centroids). Note that this algorithm only converges to a local minimum and there is no known algorithm to perform an exact decoding in polynomial time in high dimensions $P$. Moreover, the K-means algorithm cannot be readily applied to approximately maximize any linear function $\tr AM$ with respect to $M \in \mathcal{M}$, i.e., when $A$ is not positive-definite or the number of clusters is not known.

Following~\cite{ShiMalik,ng2002spectral,BachJo}, we now present a spectral relaxation of this problem. This is done by relaxing the set $\mathcal{M}$ to the set of matrices that satisfy $M^2=M$ (i.e., removing the constraint that $M$ takes a finite number of distinct values). 
When the number of clusters is known, this leads to the classical spectral relaxation, i.e., 
$$\max_{M \in \mathcal{M}, \ \tr M = K} \tr (AM) \leqslant \max_{M^2 = M, \ \tr M = K } \tr(AM) ,$$ 
which is equal to the sum of the $K$ largest eigenvalues of~$A$; the optimal matrix $M$ of the spectral relaxation is the orthogonal projector on the eigenvectors of $A$ with $K$ largest eigenvalues.

When the number of clusters is unknown, we have:
$$\max_{M \in \mathcal{M}} \tr (AM) \leqslant \max_{M^2 = M } \tr(AM) =
\tr (A)_+,$$ 
where $\tr(A)_+$ is the sum of positive eigenvalues of $A$. The optimal matrix $M$ of the spectral relaxation is the orthogonal projector on the eigenvectors of $A$ with positive eigenvalues. Note that in the formulation from Eq.~(\ref{eq:partition2}), this corresponds to thresholding all eigenvalues of $XX^\top$ which are less than $\lambda$.

We denote by $\mathcal{M}^{\rm spec} = \{ M \in \mathbb{R}^{P \times P} , \ M^2 = M\}$ and 
$\mathcal{M}^{\rm spec}_K = \{ M \in \mathbb{R}^{P \times P} , \ M^2 = M, \ \tr M = K \}$ the relaxed set of rescaled equivalence matrices.

\subsection{Metric learning}

In this paper, we consider   learning   a \emph{Mahalanobis metric}, which may be parameterized by a positive definite matrix $B \in \mathbb{R}^{P \times P}$. This corresponds to replacing dot-products $x_i^\top x_j$ by $x_i^\top B x_j$, and  $XX^\top$ by $X BX^\top$. Thus, when the number of cluster is known, this corresponds to  
\BEQ
\label{eq:BB}
\min_{M \in \mathcal{M}_K} \tr \big[ X B X^\top  ( I - M)  \big]    
\EEQ
 or, when the number of clusters is unknown, to:
\BEQ
\label{eq:AA}
 \min_{M \in \mathcal{M}}  \tr \big[  B X^\top  ( I - M) X \big] + \lambda \tr M. 
 \EEQ
Note that by replacing $B$ by $B \lambda$ and dividing the equation by $\lambda$, we may use an equivalent formulation of Eq.~(\ref{eq:AA}) with $\lambda=1$, that is:
\begin{equation}\label{eq:metlearn}
 \min_{M \in \mathcal{M}}  \tr \big[  XB X^\top  ( I - M)  \big] +   \tr M. 
\end{equation}
The key aspect of the partitioning problem is that it is formulated as optimizing with respect to $M$ a function \emph{linearly} parameterized by $B$. The linear parametrization in $M$ will be useful when defining proper losses and efficient loss-augmented inference in Section~\ref{sec:structure}.

Note that we may allow $B$ to be just positive semi-definite. In that case, the zero-eigenvalues of the pseudo-metric corresponds to irrelevant dimensions. That means in particular we have performed dimensionality reduction on the input data. We propose a simple way to encourage this desirable property in Section~\ref{Low-rank}.

\section{Loss between partitions} \label{perte}

Before going further and apply the framework of Structured prediction \cite{Tsochan} in the context of metric learning, we need to find a loss on the output space of possible partitioning which is well suited to our context. To avoid any notation conflict, we will refer in that section to $\mathcal{P}$ as a general set of partition (it can corresponds for instance to $\mathcal{M}^{\rm seq}$).
\subsection{Some standard loss}\label{sec:losses}

\paragraph{The Rand index}
When comparing partitions \cite{partitions}, a standard way to measure how different two of them are is to use the Rand \cite{Rand} index which is defined, for two partitions of the same set of $T$ elements $S$ $P_1=\lbrace P^1_1,\ldots, P^{K_1}_1 \rbrace$ and $P_2\lbrace P^1_2,\ldots, P^{K_2}_2 \rbrace$ as the sum of concordant pairs over the number of possible pairs. More precisely, if we consider all the possible pairs of elements of $S$, the concordant pairs are defined as the sum of the pairs of elements which both belong to the same set in $P_1$ and $P_2$ and of the pairs which are not in the same set both in $P_1$ and $P_2$. In matricial terms, it is linked to the Frobenius distance between the equivalence matrices representing $P_1$ and $P_2$ (these matrices are binary matrices of size $T \times T$ which are 1 if and only if the element $i$ and the element $j$ belong to the same set of the partition).\\
This loss is not necessarily very well suited to our problem, since intuitively one can see that it doesn't take into account the size of each subset inside the partition, whereas our concern is to optimize intra class variance which is a rescaled indicator.

\paragraph{Hausdorff distance}
In the change-point detection litterature, a very common way to measure dissimilarities between partitions is the so-called Hausdorff distance \cite{Boysen} on the elements of the frontier of the elements of the partitions (the need for a frontier makes it inapplicable directly to the case of general clustering). 
Let's consider two partitions of a finite set $S$ of T elements. We assume that the elements have a sequential order and thus elements of partitions $P_1$ and $P_2$ have to be contiguous.
It is then possible to define the frontier (or set of ruptures) of $P_1$ as the collection of indexes $\partial P_1=\lbrace
\inf P^2_1,\ldots,\inf P_1^K\rbrace$. Then, by embedding the set $S$ into $[0,1]$ (it corresponds just to normalize the time indexes so that they are in $[0,1]$), we can consider a distance $d$ on $[0,1]$, (typically the absolute value) and then define the associated Hausdorff distance $d_H(P_1,P_2)=\max \lbrace \sup_{x \in \partial P_1} \inf_{y \in \partial P_2} d(x,y), \sup_{y \in \partial P_2} \inf_{x \in \partial P_1} d(x,y) \rbrace$

\paragraph{The loss considered in our context}
In this paper, we consider the following loss, which was originated proposed in a slightly different form by \cite{partitions} and has then been widely used in the field of clustering \cite{BachJo}.
This loss is a variation of the $\chi^2$ association in a $K_1 \times K_2$ contingency table (see \cite{partitions}). More precisely, if we consider the contingency table associated to $P_1$ (partition of a set of size $T$) with $K_1$ elements and $P_2$ with $K_2$ elements (the contingency table being the $K_1\times K_2$ table $C$ such that $C_{i,j}=n_{ij}$ the number of elements in element $i$ of $P_1$ and in element $j$ of $P_2$), we have that $\|M-N\|^2_F=K_1+K_2-\frac{\chi^2(C)+T}{T}$.
\BEQ
 \label{losses}
\!\!\!\!\!\! \ell (M,N) = \frac{1}{T}\| M \!-\! N \|_F^2 
= \frac{1}{T}\big(\Tr(M) + \Tr(N) - 2\Tr(MN )\big).
\EEQ
Moreover, if the partitions encoded by $M$ and $N$ have clusters
 $P^1_1,\dots,P^{K_1}_1$ and  $P^1_2,\dots,P^{K_2}_2$, then 
$
T\ell (M,N)  = K_1+K_2-2 \sum_{k,l}\frac{|A_k \cap B_l|^2}{|A_k| \cdot |B_l|}.
$
This loss is equal to zero if the partitions are equal, and always less than $\frac{1}{T}(K+L-2)$.
Another equivalent interpretation of this index is given by, with the usual convention that for the element of $S$ indexed by $i$ $ P_1(i)$ is the subset of $P_1$ where $i$ belongs:
$$
T\ell (M,N)  = K_1+K_2-2 \sum^{T}_{i=1}\frac{\vert P_1(i) \cap P_2(i)\vert}{|P_1(i)| \times |P_2(i)|}.
$$

This index seems intuitively much more suited to the study of the problem of variance minimization since it involves the rescaled equivalence matrices which parametrize naturally these kind of problems.
We examine in the Appendix more facts about these losses and their links, especially about the asymptotic behaviour of the loss we use in the paper. We also show a link between this loss and the Hausdorff in the case of change-point detection.

\section{Structured prediction for metric learning}
\label{sec:structure}

As shown in the previous section, our goal is to learn a positive definite matrix $B$, in order to improve the performance of structured output algorithm that minimizes with respect to $M \in \mathcal{M}$, the following cost function
of Eq. \ref{eq:metlearn}.
Using the change of variable described in the table below, the partitioning problem may be cast as 
\[ 
\max_{M \in \mathcal{M}}~\langle w, \varphi(X,M) \rangle
\mbox{ or } \max_{M \in \mathcal{M}_K}~\langle w, \varphi(X,M) \rangle. \]
where $\langle A, B \rangle$ is the Frobenius dot product.
\renewcommand{\arraystretch}{1.12}

\vspace*{-.2cm}

\begin{center}
\label{variants}
\begin{tabular}{|c|c|c|}

\hline 
Number of clusters & $\varphi(X,M)$ & $w$
\\ \hline 
 Known & $X^{\top} M X$ & $B$        \\
 $(\tr M = K)$ & & 
\\ \hline 
Unknown & $\frac{1}{T}
\begin{pmatrix}
\ X^{\top} M X & 0 \\
0 & M
\end{pmatrix}$  & $\begin{pmatrix}
B & 0 \\
0 & -I 
\end{pmatrix}
$    \\ \hline 
\end{tabular}
\end{center}

We denote by $\mathcal{F}$ the vector space where the vector $w$ defined above belongs to.
Our goal is thus  to estimate $w \in \mathcal{F} $ from $N$ pairs of observations $(X_i,M_i) \in \mathcal{X} \times \mathcal{M}$. This is exactly the goal of large-margin structured prediction~\cite{Tsochan}, which we now present. We   denote by $\mathcal{N}$ a generic set of matrices, which may either be $\mathcal{M}$, $\mathcal{M}^{\rm spec}$, $\mathcal{M}^{\rm seq}$, $\mathcal{M}_K$, $\mathcal{M}_K^{\rm spec}$, $\mathcal{M}_K^{\rm seq}$, depending on the situation (see Section~\ref{sec:N} for specific cases).

\subsection{Large-margin structured output learning}
In the margin-rescaling framework of~\cite{Tsochan}, using a certain loss $\ell: \mathcal{N} \times \mathcal{N} \to \mathbb{R}_+$ between elements of $\mathcal{N}$ (here partitions), the goal is to minimize with respect to $ w \in \mathcal{F}$,

\vspace*{-.25cm}

$$\frac{1}{N} \sum_{i=1}^N \ell \big(  \argmax_{M \in \mathcal{N}} \langle w, \varphi(X_i,M) \rangle , M_i \big) + \Omega(w),
$$
where $\Omega$ is any (typically convex) regularizer. This framework is standard in machine learning in general and metric learning in particular (see e.g, \cite{JainKulis}).
This loss function $w \mapsto \ell \big(  \argmax_{M \in \mathcal{N}} \langle w, \varphi(X_i,M) \rangle , M_i \big)$ is not convex in $M$, and may be replaced by the convex surrogate

\vspace*{-.4cm}

$$L_i(w) =  \max_{M \in \mathcal{N}}  \big\{
\ell(M,M_i) + \langle w, \varphi(X_i,M)  -  \varphi(X_i,M_i) \rangle \big\},
$$
leading to the minimization of
\begin{equation}\label{eq:objective}
\frac{1}{N} \sum_{i=1}^N L_i(w) + \Omega(w).
\end{equation}
In order to apply this framework, several elements are needed: (a) a regularizer $\Omega$, (b) a loss function $\ell$,  and (c) the associated efficient algorithms for computing~$L_i$, i.e., solving the \emph{loss-augmented inference} problem 
$\max_{M \in \mathcal{N}}  \big\{
\ell(M,M_i) + \langle w, \varphi(X_i,M)  -  \varphi(X_i,M_i) \rangle \big\}$.

As discussed in Section \ref{perte}, a natural loss on our output space is given by the Frobenius norm of the rescaled equivalence matrices associated to partitions.

%
%

\subsection{Loss-augmented inference problem}
\label{sec:N}
Efficient minimization is key to the applicability of large-margin structured prediction and this problem  is a classical computational bottleneck. In our situation the cardinality of $\mathcal{N}$ is exponential, but the choice of loss between partitions
lead to the problem $\max_{M \in \mathcal{N}}\Tr(A_iM)$ where:
\begin{list}{\labelitemi}{\leftmargin=1.1em}
   \addtolength{\itemsep}{-.3\baselineskip}
\item[--] $A_i = \frac{1}{T} (  X_i BX_i^{\top} - 2 M_i+{\rm Id})$ if the number of clusters is known.
\item[--] $A_i = \frac{1}{T} (  X_i BX_i^{\top} - 2 M_i )$ otherwise.
\end{list}
Thus, the loss-augmented problem may be performed for the change-point problems exactly (see
Section~\ref{sec:dp}) or through a spectral relaxation otherwise (see Section~\ref{sec:spectral}). Namely,
for change-point detection problems, $\mathcal{N}$ is either $\mathcal{M}^{\rm seq}$ or $\mathcal{M}^{\rm seq}_K$, while for general partitioning problems, it is 
either $\mathcal{M}^{\rm spec}$ or $\mathcal{M}^{\rm spec}_K$.

\subsection{Regularizer}

We may consider several parametrizations/regularizers for our positive semidefinite matrix $B$. We may classically (see e.g, \cite{JainKulis}) penalize $\tr B^2 = \| B \|_F^2$, which is the classical squared Euclidean norm.
However, two variants of our algorithm are often needed for practical problems.

\vspace*{-.25cm}

 \paragraph{Diagonal metric.} \hspace*{.000cm}
To limit the number of parameters, we may be interested in only reweighting the different dimensions of the input data, i.e., we can impose the metric to be diagonal, i.e, $B=\Diag(b)$ where $b \in \mathbb{R}^P$. Then, the constraint is $ b\geqslant 0$, and we may penalize by $ \|b \|_1 = \mathbf{1}_P^\top b $ or $\|b \|_2^2$, depending whether we want to promote zeros in $b$ (i.e., to do feature selection).

\vspace*{-.25cm}

\paragraph{Low-rank metric.}  \hspace*{.000cm}\label{Low-rank}
Another potentially desirable property is the interpretability of the obtained metric in terms of its eigenvectors. Ideally we want to have a pseudo-metric with a small rank. As it is classically done, we relaxed it into the sum of singular values. Here, since the matrix $B$ is symmetric positive definite, this is simply the trace $\tr(B)$.

\subsection{Optimization}\label{sec:Optimization}

In order to optimize the objective function of Eq. \eqref{eq:objective}, we can use several optimization techniques. This objective present the drawback of being non-smooth and thus the convergence speed that we can expect are not very fast.
\\
In the structured prediction litterature, the most common solvers are based on cutting-plane methods (see \cite{Tsochan}) which can be used in our case for small dimensional-problem (i.e., low $P$). Otherwise we use  a projected subgradient method, which leads to more numerous but cheaper iterations. 
Cutting plane and Bundle methods~\cite{bundle} shows the best speed performances when the dimension of the feature space of the data to partition is low, but were empirically outperformed by a subgradient in the very high dimensional setting.

\section{Extensions}
\label{sec:extensions}

We now present extensions which make our metric learning more generally applicable.

\subsection{Spectral clustering and normalized cuts}
\label{sec:nc}
\label{sec:ncuts}
Normalized cut segmentation is a graph-based formulation for clustering aiming at finding roughly balanced cuts in graphs~\cite{ShiMalik}. The input data $X$ is now replaced by a similarity matrix
$W \in \mathbb{R}^{T \times T}_+$ and, for a known number of clusters $K$, as shown by \cite{ng2002spectral,BachJo}, it is exactly equivalent to 
$$
\max_{M \in \mathcal{M}_K} \tr \croch{M   \widetilde{W} },
$$
 where $\widetilde{W} = \Diag( W \mathbf{1})^{-1/2} W \Diag( W \mathbf{1})^{-1/2}$ is the normalized similarity matrix.

\vspace*{-.25cm}

\paragraph{Parametrization of the similarity matrix $W$.} \hspace*{.000cm}
Typically, given data points $x_1,\dots,x_T \in \mathbb{R}^P$ (in image segmentation problem, these are often the concatenation of the positions in the image and local feature vectors), the similarity matrix is  computed as
\begin{equation}
\label{eq:W}
(W_B)_{ij} = \exp \big( - (x_i - x_j)^\top B (x_i - x_j)  \big),
\end{equation}
where $B$ is a positive semidefinite matrix. Learning the matrix $B$ is thus of key practical importance.

However, our formulation would lead to efficiently learning (as a convex optimization problem) parameters only for a linear parametrization of $\widetilde{W}$. While the linear combination is attractive computationally, we follow the experience from the supervised setting where learning linear combinations of kernels, while formulated as a convex problem, does not significantly improve on methods that learn the metric within a Gaussian kernel  with non-convex approaches~(see, e.g.,~\cite{gehler2009feature,marszalek2007learning}).

We thus stick to the parametrization of Eq.~(\ref{eq:W}). In order to make the problem simpler and more tractable, we consider spectral clustering directly with $W$ and not with its normalized version, 
i.e., our partitioning problem becomes
$$
\max_{M \in \mathcal{M}} \tr W M \mbox{ or } \max_{M \in \mathcal{M}_K} \tr W M.
$$
In order to solve the previous problem, the spectral relaxation outlined in Section~\ref{sec:spectral} may be used, and corresponds to computing the eigenvectors of $W$ (the first $K$ ones if $K$ is known, and the ones corresponding to eigenvalues greater than a certain threshold otherwise).

\vspace*{-.25cm}

\paragraph{Non-convex optimization.} \hspace*{.000cm}
In our structured output prediction formulation, the loss function for the $i$-th observation becomes (for the case where the number of clusters is known):
\BEAS
 & & \max_{M \in \mathcal{M}^{\rm spec}_K }  \big\{
\ell(M,M_i) +  \tr W_B ( M - M_i)    \big\} \\
&  = &  - \tr W_B M_i +  \max_{M \in \mathcal{M}^{\rm spec}_K }  \big\{
\ell(M,M_i) +  \tr W_B   M     \big\} .
\EEAS
It is not a convex function of $B$, however, it is a difference of a concave and a convex function, which can be dealt with using majorization-minimization algorithm~\cite{yuille2003concave}. The idea of this algorithm is simply to upper-bound
the concave part $ - \tr W_B M_i$
by its linear tangent. Then the problem becomes convex and can be optimized using one of the method proposed in Section~\ref{sec:Optimization}
 We then iterate the process, which is known to be converging to a stationary point.

\subsection{Partial labellings}
\label{sec:partial}
The large-margin convex optimization framework relies on fully labelled datasets, i.e., pairs $(X_i,M_i)$ where $X_i$ is a dataset and $M_i$ the corresponding rescaled equivalence matrix. In many situations however, only partial information is available. In these situations, starting from the PCA metric, we propose to iterate between (a) label all datasets using the current metric and respecting the constraints imposed by the partial labels and (b) learn the metric using Section~\ref{sec:structure} from the fully labelled datasets.
See an application in Section~\ref{sec:experiments2}.

\subsection{Detecting changes in distribution of temporal signals} \label{sec.appli.change-distrib}
\label{sec:dist}
In sequential problems, for now, we are just able to detect changes in the mean of the distribution of time series but not to detect change-points in the whole distribution (e.g., the mean may be constant but the variance piecewise constant).
Let us consider a temporal series $X$ in which some breakpoints occur in the distribution of the data.
From this single series, we   build several series permitting to detect these changes, by considering features built from $X$, in which the change of distribution   appears as a change in mean.
A naive way would be to consider the moments of the data $X, X^2, X^3,\dots,X^r$ but unfortunately as $r$ grows these moments explode. 
A way to prevent them from exploding is to use the robust Hermite moments \cite{robustat}. 
These moments are computed using the Hermite functions and permit to consider the $p$-dimensional series $H_1(X),H_2(X), \dots$, where $H_i(X)$ is the $i$-th Hermite function $H_i(x)=2\sqrt{2^i\pi i!}e^{-\frac{x^2}{2\sigma^2}}(-1)^i2^{i/2}e^{\frac{x^2}{2}}\frac{d^i}{dx^i} \big( e^{\frac{-x^2}{2}} \big)$.

\vspace*{-.25cm}

\paragraph{Bioinformatics application.} \hspace*{.000cm}
Detection of change-points in DNA sequences for cancer prognosis
provides a natural testbed for this approach. 
Indeed, in this field, researchers face data which are linked to the number of copies of each gene along the DNA (a-CGH data as used in \cite{Toby}). 
The presence of such changes are generally related to the development of certain types of cancers.
On the data from the Neuroblastoma dataset~\cite{Toby}, some caryotypes with changes of distribution were manually annotated. Without any metric learning, the global error rate in change-point identification is 12\%. By considering the first 5 Hermite moments and  learning a metric, we reach a rate of 6.9\%, thus improving significantly the performance.

\section{Experiments}
\label{sec:experiments}

We have conducted a series of experiments showing improvements of our large-margin metric learning methods over previous metric learning techniques.

\subsection{Change point detection}\label{sec:experiments2}

\paragraph{Synthetic examples and robustness to lack of information.} \hspace*{.000cm}
 We consider $300$-dimensional time series of length $T=600$ with an unknown number of breakpoints. 
Among these series only 10 are relevant to the problem of change-point detection, i.e.,  290 series have abrupt changes which should be discarded. Since the identity of the 10 relevant time series is unknown,  by learning a metric we hope to obtain high weights on the relevant series and small weights on the others. 
The number of segments is not assumed to be known and is learned automatically.

Moreover, in this experiment we progressively remove information, in the sense that as input of the algorithm we only give a fraction of the original time series (and we measure the amount of information given through the ratio of the given temporal series compared to the original one). Results are presented in
Figure~\ref{fig:robustness}. As expected, the performance without metric learning is bad, while it is improved with PCA.
Techniques such as RCA~\cite{bar2006learning} which use the labels improve even more (all datasets were stacked into a single one with the corresponding supervision); however, it is not directly adapted to change-point detection, it requirse dimensionality reduction to work and the performance is not robust to the choice of the number of dimensions. Note also that all methods except ours are given the exact number of change-points.
Our large-margin approach outperforms the other metric, in the convex setting (i.e., extreme right of the curves), but also in partially-supervised setting where we use the alternative approach describe in Section~\ref{sec:partial}.

\begin{figure}

\vspace*{-.5cm}

\begin{center}
\includegraphics[width= .9\linewidth]{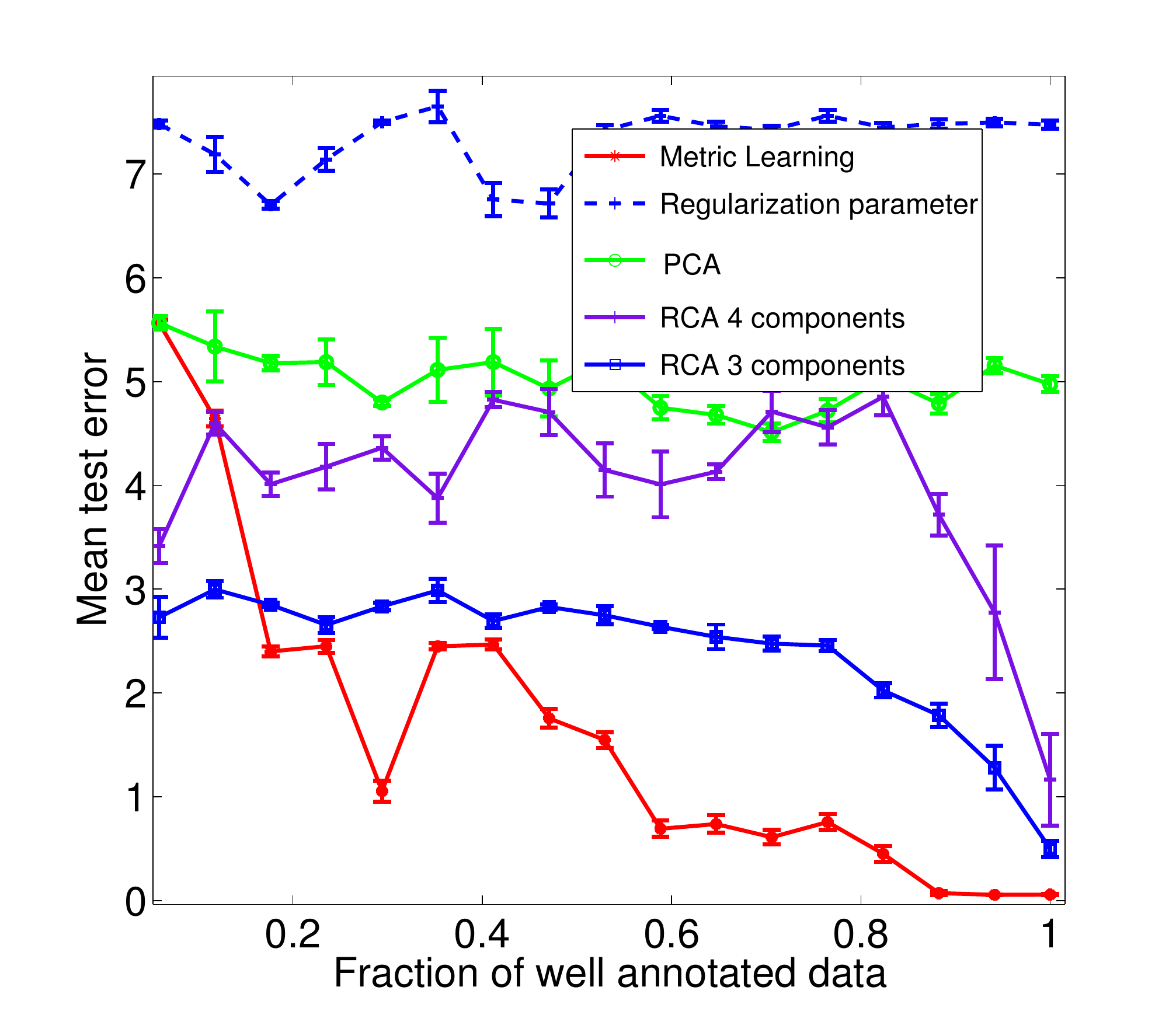}
\end{center}

\vspace*{-.45cm}

\caption{Performances on synthetic data vs.~the quantity of information available in the time series. Note the small error bars. We compare ourselves against a metric learned by RCA (with 3 or 4 components), an exhaustive search for one regularization parameter, and PCA. \label{fig:robustness}}
\end{figure}

\vspace*{-.25cm}

\paragraph{Video segmentation.}  \hspace*{.000cm}
\label{sec:video}
We applied our method to data coming from old TV shows (the length of the time series in that case is about 5400, with 60 to 120 change-points) where some speaking passages alternate with singing ones.  The videos are  from 1h up to 1h30 long.
We aim at recovering the segmentation induced by the speaking parts and the musical ones. Following ~\cite{Sylvian}, we use GIST features for the video part and MFCC features for the audio. The features were aggregated every second so that the temporal series we are considering are  about several thousands vectors long, which is still computationally tractable using the dynamic programming of Algorithm \ref{alg3}. 
 We used 4 shows for train, 3 for validation, 3 for test. The running times of our Matlab implementation were  in order of a few hours.
 
The results are described in Table~\ref{tableCPD}.   
  We consider three different settings: using only the image stream, only the audio stream or both. In these three cases, we consider using the existing metric (no learning), PCA, or our approach. In all settings,  metric learning improves performance. Note that the performance is  best with only the audio stream and our metric learning, given both streams, manages to do almost as well as with only the audio stream, thus illustrating the robustness of using metric learning in this context.

\begin{table}   
   \caption{\label{tableCPD} Empirical performance on each of the three TV shows used for testing. Each subcolumn stands for a different TV show. The smaller the loss is, the better the segmentation is. }

\vspace*{-.5cm} 

\begin{center}
\resizebox{\linewidth}{!}{
   \begin{tabular}{ |l | c |c|c| c|c|c | c|c|c| c}

     \hline
     Method &\multicolumn{3}{|c|}{Audio}
  & \multicolumn{3}{|c|}{Video} & \multicolumn{3}{|c|}{Both} \\ \hline
     PCA  & 23 & 41 & 34  & 40 & 55 & 25 & 29 & 53 & 37 \\ \hline
     Reg. parameter & 29 & 48 & 33 & 59 & 55 & 47 & 40 & 48&36\\ \hline
     Metric learning & $\bold{6.1}$ & $\bold{9.3}$ &$\bold{7}$ & $\bold{10}$ & $\bold{14 }$& $\bold{11}$ & $\bold{8.7}$& $\bold{9.6}$&$\bold{7.8}$\\ 

     \hline
     
   \end{tabular}}
  \end{center}

\vspace*{-.25cm}

\end{table}

\subsection{$K$-means clustering}

Using  the partition induced by the classes as ground truth, we tested our algorithm on some classification datasets from the UCI machine learning repository, using the classification information 
as partitions, following the methodology proposed by \cite{Xing}. This application of our framework is a little extreme in the sense that we assume only one partitioning as training point (i.e., $N=1$). The results are presented in Table \ref{clusteringtable}. For the ``Letters'' and ``Mov. Libras'' datasets, there are no significant differences, while for the ``Wine'' dataset, RCA is the best, and for the ``Iris'' dataset, our large-margin approach is best: even in this extreme case, we are competitive with existing techniques.

\begin{table}  
    \caption{Performance of the metric learning versus the Euclidean distance, and other metric learning algorithms such as RCA or \cite{Xing}. \label{clusteringtable} We use the loss from Eq.~\eqref{losses}.}

\vspace*{-.5cm} 
\begin{center}

\resizebox{\linewidth}{!}{
      \begin{tabular}{ |c|r@{ $\pm$ }l|r@{ $\pm$ }l|r@{ $\pm$ }l| r@{ $\pm$ }l| }
     \hline
     Dataset & \multicolumn{2}{|c|}{Ours} & \multicolumn{2}{c|}{Euclidean} & \multicolumn{2}{c|}{RCA} & \multicolumn{2}{c|}{\cite{Xing}} \\ \hline
	 Iris & $\bold{0.18}$ & $0.01$ & 0.55 & $10^{-11}$ &0.43 &0.02 & 0.30 &0.01  \\ \hline
	 Wine & 1.03 & $0.04$ &3.4 & $3.10^{-4}$ &$\bold{0.88}$ & 0.14 & 3.08&0.1\\ \hline
	 Letters & 34.5 & $0.1$ & 41.62 & $0.2$ &34.8 &0.5 &35.26 &$0.1$\\ \hline
     Mov. Libras & 14 & $1$  & 15 & $0.2$  & 22 & 2 & 15.07& 1\\ \hline
   \end{tabular}
   }
 \end{center}
\vspace*{-.25cm}

\end{table}
 
\subsection{Image Segmentation}
\label{sec:image}
We now consider learning metrics for normalized cuts and consider the Weizmann horses database \cite{Borenstein}, for which groundtruth segmentation is available. Using color and position features, we learn a metric with the method presented in Section~\ref{sec:ncuts} on 10 fully labelled images. We then test on  the remaining 318 images.

We compare the results of this procedure to a cross-validation approach with an exhaustive search on a 2D grid adjusting one parameter for the position features and one other for color ones.
The loss between groundtruth and segmentations obtained by the normalized cuts algorithm is measured either by Eq.~\eqref{losses} or the Jaccard distance. 
Results are summarized in Table~\ref{tablechwal}, with some visual examples in Figure~\ref{chwalbeaux}.
The metric learning within the Gaussian kernel significantly improves performance.
The running times of our pure Matlab implementation were in order of several hours to get convergence of the convex-concave procedure we used.

\begin{table}

\caption{\label{tablechwal} Performance of the metric learned in the context of image segmentation, comparing the result of a learned metric vs. the results of an exhaustive grid search (Grid). $\sigma$ is the standard deviation of the difference between the loss with our metric and the grid search. To assess the significance of our results, we perform t-tests whose p-values are respectively $2. 10^{-9}$ and $4.10^{-9}$.}

\vspace*{-.25cm}

\begin{center}
\resizebox{0.8\linewidth}{!}{

\begin{tabular}{|c|c|c|c|}
\cline{1-4}
 Loss used & Learned metric & Grid & $\sigma$ \\
 \cline{1-4}
 Loss of  Eq.~\eqref{losses} & 1.54 & 1.77 & 0.3\\ 
\cline{1-4}
  Jaccard distance &0.45 & 0.53 & 0.11\\ 
    \cline{1-4}

\end{tabular}}
\end{center}

\vspace*{-.5cm}


\end{table}
\begin{figure}
\includegraphics[width=\linewidth]{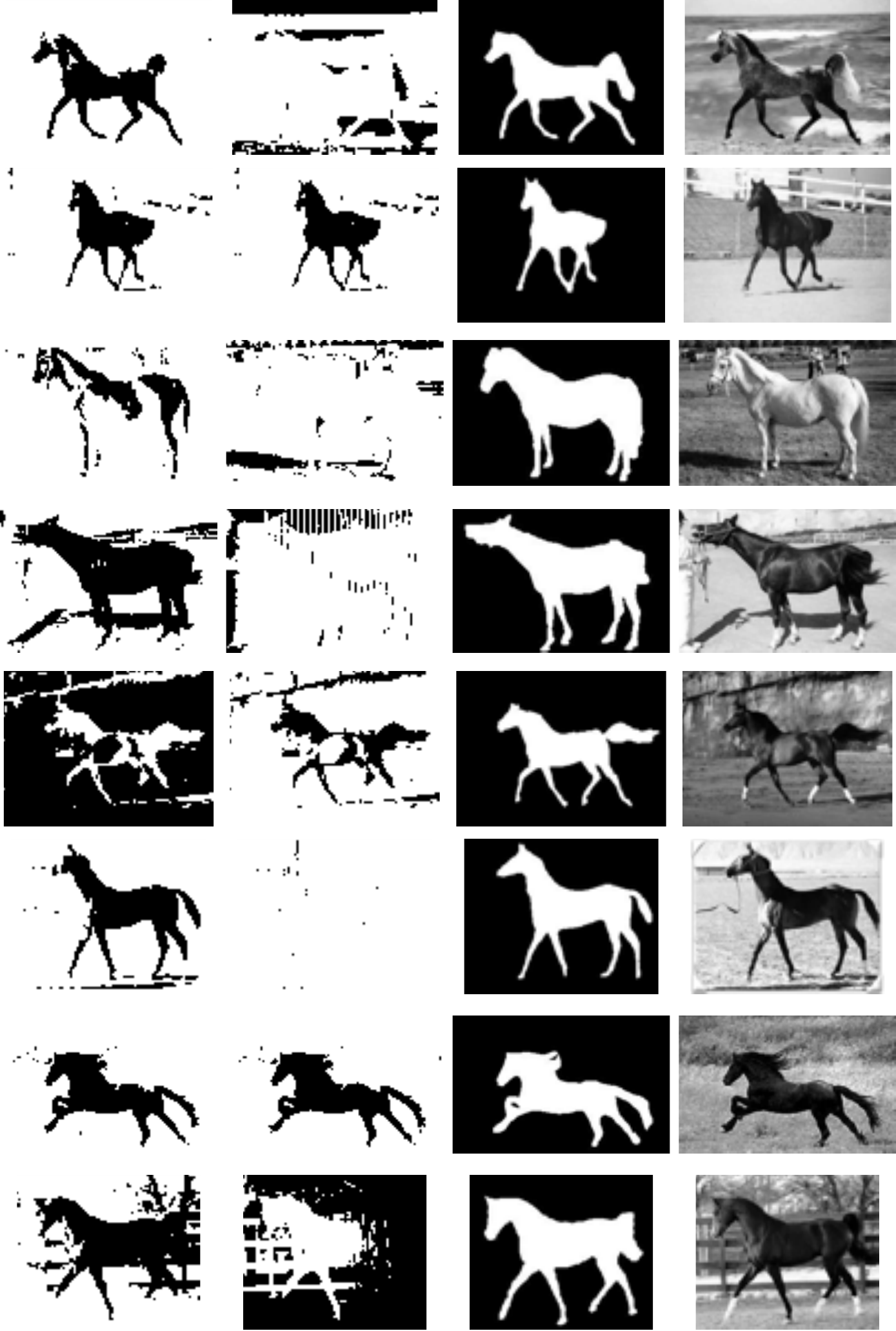}
\caption{\label{chwalbeaux}From left to right: image segmented with our learned metric, image segmented by a parameter adjusted by exhaustive search, groundtruth segmentation, original image in gray.}

\vspace*{-.4cm}

\end{figure}

\section{Conclusion}
We have presented a large-margin framework to learn metrics for unsupervised partitioning problems, with application in particular to change-point detection in video streams and image segmentation, with a significant improvement in partitioning performance. For the applicative part, following recent trends in image segmentation~(see, e.g.,~\cite{JouBacPonc10_cvpr}), it would be interesting to extend our change-point framework so that it allows unsupervised co-segmentation of several videos:  each segment could then be automatically labelled so that segments from different videos but with the same label correspond to the same action.
%

\vspace*{-.25cm}

{\small

\bibliography{Workinprogreetestbis}

\begin{thebibliography}{10}

\bibitem{Akaike}
H.~Akaike.
\newblock A new look at the statistical model identification.
\newblock {\em Automatic Control, IEEE Transactions on}, 19(6):716 -- 723, dec
  1974.

\bibitem{Sylvian}
S.~Arlot, A.~Celisse, and Z.~Harchaoui.
\newblock Kernel change-point detection, Feb. 2012.
\newblock arXiv:1202.3878.

\bibitem{BachJo}
F.~Bach and M.~Jordan.
\newblock Learning spectral clustering.
\newblock In {\em Adv. NIPS}, 2003.

\bibitem{bar2006learning}
A.~Bar-Hillel, T.~Hertz, N.~Shental, and D.~Weinshall.
\newblock Learning a mahalanobis metric from equivalence constraints.
\newblock {\em Journal of Machine Learning Research}, 6(1):937, 2006.

\bibitem{Borenstein}
E.~Borenstein and S.~Ullman.
\newblock Learning to segment.
\newblock In {\em Proc. ECCV}, 2004.

\bibitem{Boysen}
L.~Boysen, A.~Kempe, V.~Liebscher, A.~Munk, and O.~Wittich.
\newblock Consistencies and rates of convergence of jump-penalized least
  squares estimators.
\newblock {\em Annals of Statistics}, 37:157--183.

\bibitem{Caetano:Graph}
T.~S. Caetano, L.~Cheng, Q.~V. Le, and A.~J. Smola.
\newblock {Learning Graph Matching}.
\newblock In {\em IEEE 11th International Conference on Computer Vision (ICCV
  2007)}, pages 1--8, 2007.

\bibitem{chen2011parametric}
J.~Chen and A.~K. Gupta.
\newblock {\em Parametric Statistical Change Point Analysis}.
\newblock Birkh{\"a}user, 2011.

\bibitem{cheng1995mean}
Y.~Cheng.
\newblock Mean shift, mode seeking, and clustering.
\newblock {\em IEEE Trans. PAMI}, 17(8):790--799, 1995.

\bibitem{de2006discriminative}
F.~De~la Torre and T.~Kanade.
\newblock Discriminative cluster analysis.
\newblock In {\em Proc. ICML}, 2006.

\bibitem{desobry2005online}
F.~Desobry, M.~Davy, and C.~Doncarli.
\newblock An online kernel change detection algorithm.
\newblock {\em IEEE Trans. Sig. Proc.}, 53(8):2961--2974, 2005.

\bibitem{gehler2009feature}
P.~Gehler and S.~Nowozin.
\newblock On feature combination for multiclass object classification.
\newblock In {\em Proc. ICCV}, 2009.

\bibitem{goldberger2004neighbourhood}
J.~Goldberger, S.~Roweis, G.~Hinton, and R.~Salakhutdinov.
\newblock Neighbourhood components analysis.
\newblock In {\em Adv. NIPS}, 2004.

\bibitem{gower1969minimum}
J.~C. Gower and G.~J.~S. Ross.
\newblock Minimum spanning trees and single linkage cluster analysis.
\newblock {\em Applied statistics}, pages 54--64, 1969.

\bibitem{Toby}
T.~Hocking, G.~Schleiermacher, I.~Janoueix-Lerosey, O.~Delattre, F.~Bach, and
  J.-P. Vert.
\newblock Learning smoothing models of copy number profiles using breakpoint
  annotations.
\newblock {\em HAL, archives ouvertes}, 2012.

\bibitem{partitions}
L.~J. Hubert and P.~Arabie.
\newblock Comparing partitions.
\newblock {\em Journal of Classification}, 2:193--218, 1985.

\bibitem{JainKulis}
P.~Jain, B.~Kulis, J.~V. Davis, and I.~S. Dhillon.
\newblock Metric and kernel learning using a linear transformation.
\newblock {\em J. Mach. Learn. Res.}, 13:519--547, Mar. 2012.

\bibitem{JouBacPonc10_cvpr}
A.~Joulin, F.~Bach, and J.~Ponce.
\newblock Discriminative clustering for image co-segmentation.
\newblock In {\em Proc. CVPR}, 2010.

\bibitem{Lav:2005}
M.~Lavielle.
\newblock {Using penalized contrasts for the change-point problem}.
\newblock {\em Signal Proces.}, 85(8):1501--1510, 2005.

\bibitem{marszalek2007learning}
M.~Marsza{\l}ek, C.~Schmid, H.~Harzallah, and J.~Van De~Weijer.
\newblock Learning object representations for visual object class recognition.
\newblock Technical Report 00548669, HAL, 2007.

\bibitem{McFeeGert}
B.~Mcfee and G.~Lanckriet.
\newblock Metric learning to rank.
\newblock In {\em In Proceedings of the 27th annual International Conference on
  Machine Learning (ICML}, 2010.

\bibitem{ng2002spectral}
A.~Y. Ng, M.~I. Jordan, and Y.~Weiss.
\newblock On spectral clustering: Analysis and an algorithm.
\newblock {\em Adv. NIPS}, 2002.

\bibitem{Rand}
W.~M. Rand.
\newblock Objective criteria for the evaluation of clustering methods.
\newblock {\em Journal of the American Statistical Association}, 66(336):pp.
  846--850, 1971.

\bibitem{Rigaill}
G.~Rigaill.
\newblock Pruned dynamic programming for optimal multiple change-point
  detection.
\newblock Technical Report 1004.0887, arXiv, 2010.

\bibitem{ShiMalik}
J.~Shi and J.~Malik.
\newblock Normalized cuts and image segmentation.
\newblock {\em IEEE Trans. PAMI}, 22:888--905, 1997.

\bibitem{szummer2008learning}
M.~Szummer, P.~Kohli, and D.~Hoiem.
\newblock Learning {CRF}s using graph cuts.
\newblock {\em Proc. ECCV}, 2008.

\bibitem{Taskar}
B.~Taskar, C.~Guestrin, and D.~Koller.
\newblock Max-margin markov networks.
\newblock {\em Adv. NIPS}, 2003.

\bibitem{bundle}
C.~H. Teo, S.~Vishwanathan, A.~Smola, and V.~Quoc.
\newblock Bundle methods for regularized risk minimization.
\newblock {\em Journal of Machine Learning research}, 2009.

\bibitem{Tsochan}
I.~Tsochantaridis, T.~Hoffman, T.~Joachims, and Y.~Altun.
\newblock Support vector machine learning for interdependent and structured
  output spaces.
\newblock {\em Journal of Machine Learning Research}, 2005.

\bibitem{Weinberger}
K.~Q. Weinberger, J.~Blitzer, and L.~K. Saul.
\newblock Distance metric learning for large margin nearest neighbor
  classification.
\newblock In {\em Adv. NIPS}, 2006.

\bibitem{robustat}
M.~Welling.
\newblock Robust higher order statistics.
\newblock {\em Proc. Int. Workshop Artif. Intell. Statist.(AISTATS, 2005)}.

\bibitem{Xing}
E.~P. Xing, A.~Y. Ng, M.~I. Jordan, and S.~Russell.
\newblock Distance metric learning with applications to clustering with
  side-information.
\newblock {\em Adv. NIPS}, 2002.

\bibitem{yuille2003concave}
A.~Yuille and A.~Rangarajan.
\newblock The concave-convex procedure.
\newblock {\em Neural Computation}, 15(4):915--936, 2003.

\end{thebibliography}
}
\bibliographystyle{abbrv}

\renewcommand\thesection{\Alph{section}}

\setcounter{section}{0} 

\section{Asymptotics of the loss between partitions}
Note that in this section, we will denote by $d_F^2$ the ``normalized'' loss between partitions. This means that, with the notations of the article when considering two matrices $M$ and $N$ representing some partitions $P$ and $Q$ in the generic set of partitions $\mathcal{P}$, we have $Td_F^2=\|M-N\|_F^2$. Throughout this section, we will refer to the size of a partition as the number of clusters.
\subsection{Hypothesis}
\begin{itemize}

\item We assume we consider $P$ and $Q$ two partitions of the same size, with a common number of clusters $K$. 
\item $\forall k,l \in \lbrace 1,\ldots,K \rbrace$, we denote $\epsilon_{k \rightarrow l}=\vert P_k \cap Q_l \vert$, the flow which goes out from $P$ to $Q$ when $P$ goes to $Q$.
\item We define the global outer flow as $\epsilon_{k\rightarrow}=\sum_{l \neq k}\epsilon_{k \rightarrow l}$ and the global inner flow as $\epsilon_{\rightarrow l}=\sum_{l \neq k}\epsilon_{k \rightarrow l}$
\end{itemize}
\subsection{Main result}
\begin{theorem}
Let $P$ and $Q$ two partitions satisfying our hypothesis. If we note $M(P,Q)=\max_{k \neq l}\big\lbrace \frac{\epsilon_{k \rightarrow l} }{min(\vert P_k \vert \vert P_l\vert)}\big\rbrace$, then $\exists \delta:\mathcal{P}^2 \rightarrow \mathbb{R}$ such that $\sup_{P,Q, K\times M(P,Q)\leq \epsilon} \vert \delta(P,Q) \vert \rightarrow_{\epsilon \rightarrow 0} 0$ and $\forall P,Q \in \mathcal{P}$ of the same size $K$, $T $, $$T d^2_F(P,Q)=2 \sum^K_{k=1} \big ( \frac{\epsilon_{k \rightarrow}+\epsilon_{\rightarrow k}}{\vert P_k \vert}\big )\times (1+\delta(P,Q))$$
\end{theorem}

\begin{proof}
From the expressions of Section \ref{sec:losses}, we can write :
\begin{eqnarray*}
d^2_F(P,Q)&=& 2K-2 \sum_{k,l}\frac{\vert P_k \cap Q_l \vert^2}{\vert Q_l \vert\vert P_k\vert}\\
&=&2\sum^K_{k=1}(1-\frac{\vert P_k \cap Q_k \vert^2}{\vert Q_k \vert\vert P_k\vert})+2 \sum_{k \neq l} \frac{\epsilon_{k \rightarrow l}^2}{\vert P_k \vert(\vert P_l\vert-\epsilon_{\rightarrow}+\epsilon_{\rightarrow k})}
\end{eqnarray*}
The second term can be pretty easily bounded using $M$ $$2 \sum_{k \neq l} \frac{\epsilon_{k \rightarrow l}^2}{\vert P_k \vert(\vert P_l\vert-\epsilon_{k \rightarrow}+\epsilon_{\rightarrow l})} \leq 2M \sum_{k \neq l} \frac{\epsilon_{k\rightarrow l}}{\vert P_l\vert-\epsilon_{k \rightarrow}}. $$
We can go further, noticing that $\epsilon_{k \rightarrow} \leq KM\vert P_k\vert$, which leads eventually to, if $M\leq 1/2K$ (and this is the case if $\delta$ tends to 0 in the sense of the assumption of the theorem): 

$$2 \sum_{k \neq l} \frac{\epsilon_{k \rightarrow l}^2}{\vert P_k \vert(\vert P_l\vert-\epsilon_{\rightarrow}+\epsilon_{\rightarrow k})} \leq 2M \sum_{k \neq l} \frac{\epsilon_{k\rightarrow l}}{\vert P_l\vert-\epsilon_{k \rightarrow}} \leq 4M \sum_{k \neq l}\frac{\epsilon_{k  \rightarrow l}}{P_l}.$$
\\

Now, let's bound the first term, which is a little more long:

\begin{eqnarray*}
1-\frac{\vert P_k \cap Q_k \vert^2}{\vert Q_k \vert\vert P_k\vert}&=&1-\frac{(\vert P_k -\epsilon_{k \rightarrow})^2}{\vert P_k (\vert P_k\vert-\epsilon_{k \rightarrow}+\epsilon_{ \rightarrow k})}\\
&=&\frac{\epsilon_{k \rightarrow}+\epsilon_{ \rightarrow k}}{\vert P_k \vert} \times \big ( \frac{1}{1+\frac{-\epsilon_{k \rightarrow}+\epsilon_{ \rightarrow k}}{\vert P_k \vert}}\big )-\frac{\epsilon_{k \rightarrow}^2}{\vert P_k \vert(\vert P_k \vert-\epsilon_{k \rightarrow}+\epsilon_{ \rightarrow k})}
\end{eqnarray*}

But, for the same reasons as when we bounded the second term
$$\frac{\epsilon_{k \rightarrow}^2}{\vert P_k \vert(\vert P_k \vert-\epsilon_{k \rightarrow}+\epsilon_{ \rightarrow k})}\leq 2 \sum^K_{k=1} \frac{\epsilon^2_{k \rightarrow}}{\vert P_k \vert^2}.$$
Using the fact that $\forall k, (K)M\geq \frac{\epsilon_{k \rightarrow}}{\vert P_k \vert}$, we finally get that, when $M\leq 1/2K$:
$$\frac{\epsilon_{k \rightarrow}^2}{\vert P_k \vert(\vert P_k \vert-\epsilon_{k \rightarrow}+\epsilon_{ \rightarrow k})}\leq 4M\sum^K_{k=1} \frac{\epsilon_{k \rightarrow}}{\vert P_k \vert}.$$

Thus, putting everything together, when $KM \rightarrow 0$, we get the statement of the theorem.
\end{proof}

\section{Equivalence between the loss between partition and the Hausdorff distance for change point detection}

As mentioned in the title of this , there is a deep link between the Hausdorff distance and the distance between partition we used throughout this paper in the case of change-point detection applications.
We propose here to show that the two distances are equivalent.
\subsection{Hypothesis and notations}
\begin{itemize}
\item We consider the segmentations $P$ and $Q$ has having been embedded in $[0,1]$ so that we can consider a distance $d$ on $[0,1]$ to define the Hausdorff distance between the frontiers of the elements of $P$ and $Q$.
\item We denote $l_{m}(P)$ the minimal length of a segment in a partition $P \in \mathcal{P}$ and $l_{ma}$ the maximal one.
\item We denote by $d_h$ the Hausdorff distance between partitions as described in Section \ref{perte}
\end{itemize}
\subsection{Main result}
\begin{theorem}
Let P,Q denote two partitions.
If $\vert P \vert=\vert Q \vert$ and $d_h(P,Q)=\epsilon < \frac{1}{2} l_m(P)$, then we have the following:
$$\frac{\epsilon}{l_{ma}(P)}\leq d^2_F(P,Q) \leq 12 K \frac{\epsilon}{l_m(P)}.$$
Moreover, without assuming $\vert P \vert=\vert Q \vert$, we get
$$d^2_F(P,Q)\geq \frac{\epsilon}{\max(l_{ma}(P),l_m(Q)))}\geq \frac{\epsilon}{T}$$
\end{theorem}

\begin{proof}

First, let's do the majorization part
Using the expressions of Section \ref{sec:losses}, we have to minorate $\sum^K_{k,l =1}\frac{|P_k \cap Q_l|^2}{|P_k||Q_l|}$. Note that the hypothesis of the Hausdorff distane being inferior to the half of the minimal length is just here to say that the $l$-th segment of partition Q can only overlap with $l-1$-th, $l$th and $l+1$-th elements of $P$. Thus :
\begin{eqnarray*}
\sum^K_{k,l =1}\frac{|P_k \cap Q_l|^2}{|P_k||Q_l|}&=&\sum^K_{k,=1}\frac{|P_k \cap Q_k|^2}{|P_k||Q_k|}+\sum^{K-1}_{k,=1}\frac{|P_k \cap Q_{k+1}|^2}{|P_k||Q_{k+1}|}+\sum^{K-1}_{k=0}\frac{|P_k \cap Q_{k-1}|^2}{|P_k||Q_{k-1}|}\\
&\geq & \sum^K_{k=1} \frac{(|P_k|-2\epsilon)^2}{|P_k|+2 \epsilon} \\
& = & \sum^K_{k=1} \frac{((1-2\frac{\epsilon}{|P_k|})^2}{1+2\frac{\epsilon}{|P_k|}} \\
&\geq & K-6 \epsilon \sum^K_{k=1}\frac{1}{|P_k|}\\
&\geq  & K-6\frac{\epsilon K}{l_m(P)}\\
\end{eqnarray*}
which gives us the majorization.
Note that we used the fact that $\forall x \in [0,1], $ the inequality $\frac{(1-x)^2}{1+x}\geq 1-3x$ holds.
\\

For the minoration, note that it is true all the time, but we will just give the proof in the case where the Hausdorff distance is such that $d_h(P,Q) \leq l_m(P)/2$ and where $|P|=|Q|$.

First, let's begin by some general statements :\\
i)By definition $\epsilon=\max\lbrace \max_{\bar{P_i} \in \partial P}\min_{\bar{Q_j} \in \partial Q}d(\bar{P_i},\bar{Q_j})  \max_{\bar{Q_i} \in \partial Q}\min_{\bar{P_j}  \in \partial P}d(\bar{Q_i},\bar{P_j})\rbrace$.\\
ii) If the first term in the max is attained, that means there exists some $(i^*,j^*)$ such that $|\bar{P}_{i^*}-\bar{Q}_{j^*}|=\epsilon$. It also means that, if we look at the sequences, there is no elements of $\partial Q$ is between $\bar{P}_{i^*}$ and $\bar{Q}_{j^*}$. Thus, by definition of the loss $d_F^2(P,Q) \geq 2\sum_{\alpha \in P_j \cap Q_{j^*-1}, \beta \in P_j \setminus Q_{j^*-1}}\frac{1}{P_j^*}^2$, and a short computation leads to $d^2_F(P,Q)\geq 2\frac{\epsilon}{|P_j^*|}(1-\frac{\epsilon}{|P_{j^*}|})_+$.\\
iii) If the second term in the max is attained, the same minoration holds by permuting indices.
\\
Let's go back to our special case, we have $|P^*_i|>2 \epsilon$ and $|Q^*_i|>2 \epsilon$.
This leads to $$d^2_F(P,Q) \geq \max(\frac{\epsilon}{l_{ma}(P)},\frac{\epsilon}{l_{ma}(Q)})$$
\end{proof}
%
\end{document}